\documentclass{article}

\usepackage{microtype}
\usepackage{graphicx}
\usepackage{subfigure}
\usepackage{booktabs} 

\usepackage{hyperref}

\usepackage{amsthm}
\usepackage{arydshln}

\usepackage{amsmath}
\usepackage{amssymb}
\newtheorem{theorem}{Theorem}[section]

\newtheorem{lemma}[theorem]{Lemma}
\newtheorem{proposition}[theorem]{Proposition}

\usepackage[accepted]{icml2020}

\icmltitlerunning{Interferometric Graph Transform}

\begin{document}

\twocolumn[
\icmltitle{Interferometric Graph Transform:\\ a Deep Unsupervised Graph  Representation}



\icmlsetsymbol{equal}{*}

\begin{icmlauthorlist}
\icmlauthor{Edouard Oyallon}{cnrs}
\end{icmlauthorlist}

\icmlaffiliation{cnrs}{CNRS, LIP6}

\icmlcorrespondingauthor{Edouard Oyallon}{edouard.oyallon@lip6.fr}

\icmlkeywords{Graphs, Deep Learning, Unsupervised}

\vskip 0.3in
]



\printAffiliationsAndNotice{}

\begin{abstract}
We propose the Interferometric Graph Transform (IGT), which is a new class of deep unsupervised  graph convolutional neural network for building  graph representations.  Our first contribution is to propose a generic, complex-valued spectral graph architecture obtained from a generalization of the Euclidean Fourier transform. We show that our learned representation consists of both discriminative and invariant features, thanks to a novel greedy concave objective.  From our experiments, we conclude that our learning procedure exploits the topology of the spectral domain, which is normally a  flaw of spectral methods, and in particular our method can recover an analytic operator for vision tasks. We test our algorithm on various and challenging tasks such as image classification (MNIST, CIFAR-10), community detection (Authorship, Facebook graph) and action recognition from 3D skeletons videos (SBU, NTU), exhibiting a new state-of-the-art in spectral graph unsupervised settings.
\end{abstract}
\section{Introduction}
\label{introduction}
Recently, a huge interest has arisen in canonical representations for non-Euclidean domains, which has lead to the development of Graph Convolutional Neural Networks (GCNN) \cite{bronstein2017geometric,kipf2016variational}. They are useful to describe many  types of data: social networks \cite{wu2019comprehensive}, manifolds \cite{henaff2015deep}, 3D-skeletons \cite{mazari2019mlgcn}, molecules \cite{de2018molgan}, and text \cite{NIPS2016_6081}, etc. Mainly two classes of architectures address this representation learning task: on one side, the spatial graph convolution methods \cite{wu2019comprehensive} which rely on node neighborhoods of a given graph, and on the other side,  spectral methods \cite{defferrard2016convolutional}, which heavily rely on  spectral representations estimated from a given Laplacian operator. We consider the latter: in this setting, typical successful representations are obtained from deep GCNN, whose filters are often learned through supervision \cite{bronstein2017geometric}. We also note that GCNNs on regular grid exhibit a significant performance gap with their Euclidean domain counterpart, indicating that more effort must be done to incorporate efficiently this geometry. In this work, we propose a new class of spectral architecture which is unsupervisedly infered from a graph's Laplacian.

By design, standard spectral methods suffer from several inherent issues, which also apply to Euclidean domain. A first issue is the  lack of topology of the Laplacian's eigenvectors. For the sake of illustration, observe that for a smooth $f\in L^2(\mathbb{R}^k),k\geq 0$, the Fourier transform of its Laplacian satisfies:

\[\forall\omega\in\mathbb{R}^k, \widehat{\Delta f}(\omega)=-\Vert \omega\Vert^2\hat f(\omega)\,.\]

Here, the topology of the eigenbasis (e.g., a cosine family) is difficult to exhibit from its corresponding eigenvalues.  For instance, two rather different frequencies (e.g., $\omega_1\neq\omega_2$) with the same amplitude (e.g., $\Vert\omega_1\Vert=\Vert\omega_2\Vert$) will not be distinguished by a spectral clustering algorithm based solely on $\Vert\omega\Vert$. This typically leads to filters which are isotropic and not selective to a specific direction, which also holds for spatial methods \cite{bronstein2017geometric}. A second issue is that  the graph convolution employs filters which are built from local operators such as a Laplacian matrix: this typically leads to a smoothing operator \cite{kampffmeyer2019rethinking,li2018deeper,nt2020frequency,wu2019simplifying}. Thus, in those settings, spectral GCNN lose the ability to discriminate high-frequency attributes of a signal, which are also usually unstable and thus difficult to capture \cite{mallat1999wavelet}. In our work, we address those two issues by learning a complex-valued isometry in the spectral domain, which has, for instance, the ability to recover the spectral topology of 2D frequencies, without incorporating any specific prior: the filters are anisotropic and smooth in frequencies.

With GCNN, many architecture choices and design remain unclear and are obtained from a trial and error engineering process, such as the choice of a pooling operator or the non-linearity. Contrary to this, we motivate each building-block of our architecture from the scope of obtaining smooth graph representation. Furthermore, if the operators of a deep network are trained end-to-end, they lack  interpretability (e.g., explicit layer objective) because an end-to-end optimization algorithm can specify freely the weights of the internal layers of a neural network
~\cite{Oyallon_2017_CVPR}. For the sake of analysis, we learn each layer successively via a greedy procedure \cite{belilovsky2018greedy,belilovsky2019decoupled}. 

Stability to deformations and perturbations of a graph is motivated and achieved in \citet{gama2019stability}. In our work, our representation will  clearly  not be stable to local changes in the graph metric due to deformations (see Section \ref{fourier}). Instead, we address the problem of learning invariant to permutations but discriminative features. This is achieved through a simple averaging (or smoothing) which is deduced from the graph Laplacian, and our linear operators are optimized to be discriminative and to lead to smooth features.

We denote our approach Interferometric Graph Transform (IGT). Our architecture which consists of a cascade of complex isometry, modulus non linearity and linear averaging. No supervision is needed, and our representation is guaranteed to achieve a global invariance over the permutations of the graph domain, if  the final task requires it. Unsupervised learning is of particular interest for large datasets whose labeling cost is high. We also require the adjacency matrix for learning each linear operator, because our method relies on the intrisic topology of the data, yet this could be estimated from the data themselves \cite{carey2017graph}.

The IGT is defined in Section \ref{def}. First, Section \ref{fourier} defines a generalization of the complex-valued Euclidean Fourier Transform. Then, we explain our choice of linear operator in Section \ref{isometry}, and the optimization process is described in Section \ref{proj}. Finally, Section \ref{exp} reports our accuracies at the level of the state of the art on vision, skeletons and community detection tasks, which indicates the genericity of our approach. The corresponding code can be found here: \href{https://github.com/edouardoyallon/interferometric-graph-transform}{https://github.com/edouardoyallon/interferometric-graph-transform}.

\textbf{Notations}: for some complex or real vectors $x=(x[i])_i, y=(y[i])_i$, we consider the Hermitian scalar product $\langle x,y\rangle=\sum_i x[i]\overline{y[i]}$ and we write  $\Vert x\Vert^2=\langle x,x\rangle$. Also, $\mathbf{j}^2=-1$. The operator norm of a complex or real operator is given by $\Vert W\Vert=\sup_{x}\frac{\Vert Wx\Vert}{\Vert x\Vert}$. We write $x>0$ iff $\forall i, x[i]\geq 0$ and $x\neq 0$. We also denote $\{A,B\}$ the concatenation of the operators $A,B$, and $A^*=\overline{A}^{\operatorname{T}}$, the transconjugate.
\section{Related works}
\label{related}
The Group Scattering Transform \cite{mallat2012group} is a non-linear operator which can be interpreted as a  complex neural network defined over a Euclidean space sampled from a regular grid. Similarly to our work, it corresponds to a cascade of unitary transform, complex modulus and linear averaging. Yet, the unitary operators are fixed as a dilated wavelets family and involve no learning procedures. Scattering Transforms are thus difficult to adapt to non-regular grids. To tackle this issue, \citet{gama2018diffusion,gama2019stability} introduces the Graph Scattering Networks. They consist in a cascade of real wavelet transform and absolute value non-linearity. The wavelet transform is typically defined via the eigenvectors of the graph Laplacian, and thus suffer from issues stated in the Introduction. Furthermore, an absolute value is used in order to introduce a demodulation, yet the filters are not designed to do so, contrary for instance to a Gabor transform \cite{oyallon2018compressing}.  Another comparable architecture is the Haar Scattering Network \cite{chen2014unsupervised}, which employs Haar wavelets. A Haar tree is defined by forming  pairs of nodes with similar statistics, yet this considerably reduces the class of graphs that can be represented. Another proposition was to implement unitary operator for reducing the variance \cite{mallat2013deep}, in the specific case for which the averaging is performed by block. Yet, integrating the Laplacian's knowledge in these two formulations remains unclear as well as the link with invariance. Contrary to wavelet transforms, our operators are not structured by dilated filters: we shall see that our filters are closer to a Windowed Fourier transform \cite{mallat1999wavelet}. Due to this reason, while our representation has a lot of similarity with a Scattering Transform, we decided to use a different name borrowed from \cite{mallat2010recursive}.

A large variety of pooling operator has been proposed: linear pooling \cite{bruna2013spectral,gao2019learning,gao2019graph,luzhnica2019clique}, attention based pooling \cite{lee2019self}, max-pooling \cite{hamilton2017inductive,defferrard2016convolutional}, soft-max pooling \cite{ying2018hierarchical} and more. The general principle which guides the design of those pooling operators is an intermediary step of dimensionality reduction for handling large graphs and speeding up computations. To our knowledge, this is the first work to introduce and motivate a $\ell^2$-pooling (here, a modulus non-linearity) for graphs, specifically designed for building a representation whose discriminability will be preserved after the composition with a smoothing operator.

A related line of work proposed to combine GCNNs and auto-encoders \cite{salha2019keep,kipf2016variational}. The main idea is to embed the graph representations into a lower dimensional space thanks to a reconstruction criterion. Yet, this formulation does not take in account the need of invariance for addressing  certain tasks. \citet{belilovsky2017learning} applied deep networks to learn models to infer unsupervised graph structures, this however requires restrictions on the underlying data distribution. On the other hand, \cite{wu2019simplifying} postulates that progressive linear low-pass filtering is a key ingredient responsible for the success of GCN. Yet, the obtained invariants are linear: in our work, we build non-linear invariants thanks to a modulus non-linearity.

The goal of our method is not to describe graphs, yet signals whose topology is given by a fixed graph. In other words,  the graph we use is not sample dependant. Thus, we do not compare to unsupervised lines of works such as \citet{ren2019heterogeneous,velivckovic2018deep}. Furthermore, note that one could consider the larger setting of graphons \cite{ruiz2020graphon}, that would allow more flexibility in term of graph lengths.
\section{Interferometric Graph Transform}
\label{main}
\subsection{Definition}\label{def}
Let $d \in\mathbb{N}$ be the dimension of interest. We now introduce the Interferometric  Transform, which is defined over signal $x\in \mathbb{R}^{2d+1}$, without loss in generality\footnote{Up to adding a 0 component, one can always assume it because the 0-th frequency has no meaningful pairing.}. It consists of a cascade of  linear isometries,  pointwise modulus non-linearities and linear averagings. This is similar to \citet{gama2018diffusion} yet the linear operator is not a family of dilated wavelets. Formally, for a sequence of complex linear operators $W_n$, we define recursively the real non-linear operator:
\begin{equation}
\begin{cases}
U_{n+1} x= |W_n U_nx|,\\
U_0x=x.
\end{cases}
\end{equation}
Typically, $U_nx$ is a concatenation of signals with same dimension as $x$ and the operator $W_n=\{W_n^k\}$ applies simultaneously the same  collection of filters $W_n^k$ to each element of $U_nx$. Given a linear averaging $A$, we then define the Interferometric  Transform\footnote{We recall that we did not employ wavelets, thus leading to a name different from  "Scattering Transform".} of order $N\in\mathbb{N}$ as:
\begin{equation}
S_Nx=\{AU_Nx,...,AU_0x\}\,.
\end{equation}
It is illustrated Fig. \ref{fig:igt}. We shall choose the $\{W_n\}_n$ approximatively unitary, meaning that there exists $0\leq\epsilon<1$, such that for any $x\in \mathbb{R}^{2d+1}$:
\begin{equation}
(1-\epsilon)\Vert x\Vert^2\leq\Vert Wx\Vert^2+\Vert Ax\Vert^2\leq \Vert x\Vert^2\,.\label{isometry2}
\end{equation}

\begin{figure}
    \centering
    \vspace{-8pt}
    \includegraphics[width=0.7\linewidth]{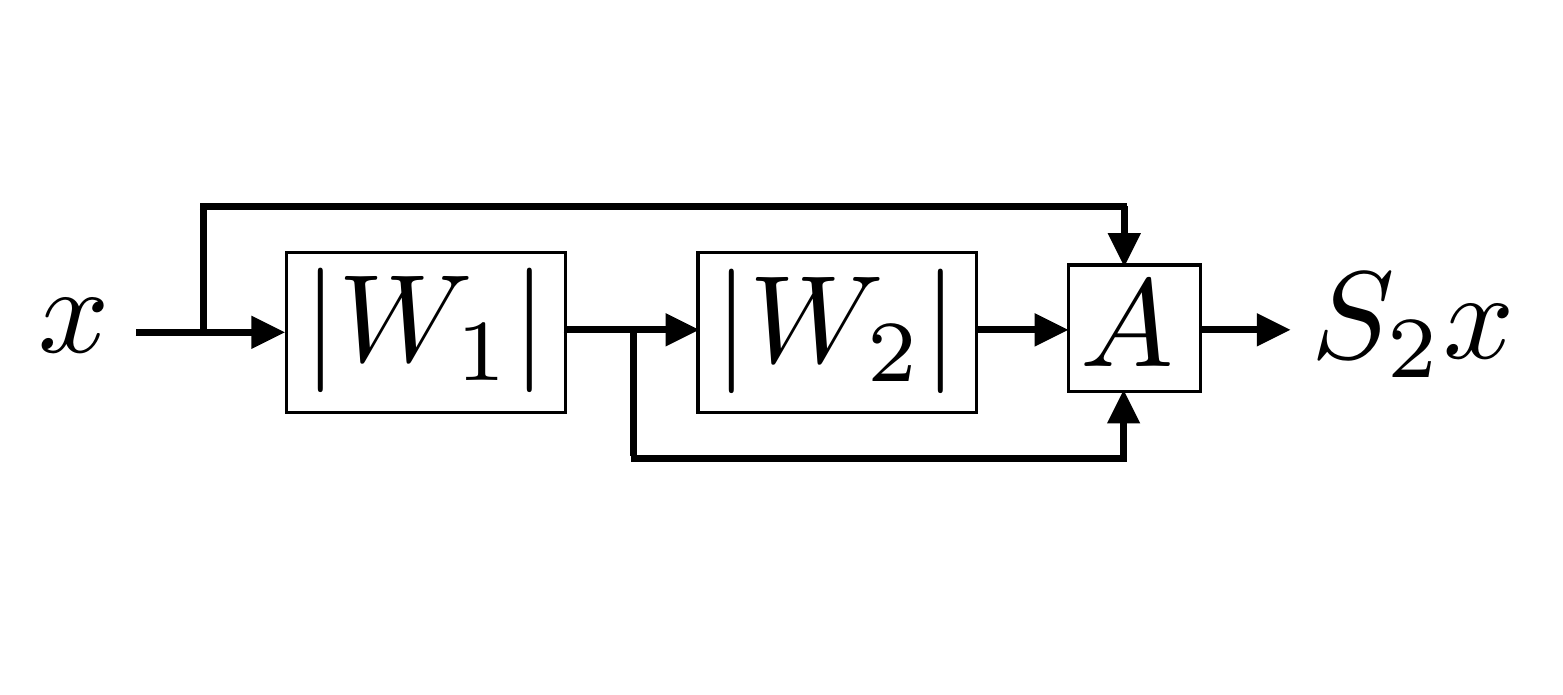}
    \vspace{-30pt}
    \caption{An illustration of the IGT with $N=2$.}
    \label{fig:igt}
    \vspace{-10pt}
\end{figure}

The operator $A$ is chosen non negative and of norm 1. The following lemma will be helpful to show that $S_N$ does preserve the energy of a signal:
\begin{lemma}\label{mean}
If $\forall x>0, Ax>0$ and $\Vert A\Vert\leq 1$, then $\exists 0< C\leq 1,  \forall x>0,$
\[C\Vert x\Vert\leq\Vert Ax\Vert\leq\Vert x\Vert\,.\]
\end{lemma}
\begin{proof}
Since $\Vert Ax\Vert\leq \Vert x\Vert$, then if such $C$ exists, it satisfies $C\leq 1$. Furthermore, $\{x>0\}\cap \{\Vert x\Vert=1\}$ is compact, thus, $x\rightarrow \Vert Ax\Vert$ is minored by $C$ and reached: there exists $C, x_0>0, \Vert x_0\Vert=1$ s.t. $\Vert Ax_0\Vert=C>0$ by assumption on $A$.
\end{proof}
In this case, Proposition \ref{norm} shows that $S$ is approximatively unitary and is non-expansive.
\begin{proposition}\label{norm}For $N\in\mathbb{N}$, $S_N$ is non-expansive, ie, $\forall x,y$:
\begin{equation}
    \Vert S_Nx-S_Ny\Vert\leq \Vert x-y\Vert\,, 
    \label{exp}
\end{equation} 
and also:
\begin{equation}
   \Vert  S_Nx\Vert\leq \Vert x\Vert\,. \label{iso}
\end{equation}
Furthermore, if $\epsilon=0$, then:
\begin{equation}\label{isosup}
   \lim_{N\rightarrow \infty} \Vert  S_Nx\Vert= \Vert x\Vert\,. 
\end{equation}
\end{proposition}
\begin{proof}
For Eq. \eqref{exp}, observe that one has a cascade of non-expansive operator. For Eq. \eqref{iso}, observe that $S_N0=0$. For the other side of the inequality, observe that Equation \eqref{isometry2} leads to:
\begin{equation}
    \Vert U_nx \Vert^2=  \Vert AU_nx\Vert^2+\Vert U_{n+1}x\Vert^2\label{eq1}
\end{equation}
Thus, a simple sum leads to:

\[\sum_{n=0}^N\Vert AU_nx\Vert^2+\Vert U_{N+1}x\Vert^2=\Vert x\Vert^2\,,\]
and from Lemma \ref{mean}, $\Vert U_Nx\Vert^2 \leq (1-C^2)^N\Vert x\Vert^2\rightarrow 0$ allows to conclude.

\end{proof}

We will now discuss the specific setting of Interferometric Transforms defined over graphs $\mathcal{G}$.
\subsection{Recovering a Fourier basis}\label{fourier}
The goal of this section is to introduce our complex operator which is equivalent to a Fourier Transform designed specifically for a graph. Here, we will consider signals whose coordinate's topology is organized by a graph $\mathcal{G}$, with $2d+1$ nodes and we name its corresponding Laplacian operator $\mathcal{L}$. Without loss of generality, we consider graphs with a single component, the extension to more components being natural by considering each sub-component individually. We write the orthogonal diagonalization basis of $\mathcal{L}$: $\{e_1,...,e_{2d+1}\}$, such that $\mathcal{L}e_{2d+1}=0$. In our applications, $e_{2d+1}=(1,...,1)$ and a typical averaging $A$  that we will use, corresponds to:
\[Ax=\langle x,e_{2d+1}\rangle\,.\]
This is a consistant choice with Lemma \ref{mean}. Note that we need to make an arbitrary choice of  basis at the moment that one eigenvalue is of multiplicity larger than 1.  As discussed in the introduction, the indexes $\{1,...,2d+1\}$ of the basis lacks structure, meaning that the order of the eigenvectors does not reflect the actual geometry of $\mathcal{L}$. However, standard approaches \cite{hammond2011wavelets} sort eigenvectors according to the amplitude of their eigenvalues and they employ this topology: here, we propose a rather different approach which behaves well in constant curvature settings. Our objective will be to pair basis' atoms according to the smoothness of their envelope. This will be analogous to form Hilbert pairs \cite{krajsek2007unified}, which corresponds to a pairing of the elements of the basis $\{e_1,...,e_{2d}\}$ in order to design an analytic representation \cite{johansson1999hilbert}.  Well localized representations using Hilbert pairs have typically a smooth modulus \cite{oyallon2018compressing}, and this analogy is a motivation for introducing this notion. To do so, for a permutation $\pi$ of $\{1,...,2d\}$, we introduce the pairing cost:
\begin{equation}
\begin{split}
C(\pi)&=\sum_{i=1}^d \sum_{k=1}^{2d+1} \sqrt{e_{\pi[2i]}[k]^2+e_{\pi[2i-1]}[k]^2}\\
&=\sum_{i=1}^d \Vert e_{\pi[2i]}+\mathbf{j}e_{\pi[2i-1]}\Vert_1\,.
\end{split}
\end{equation}

Observe that a simple application of Cauchy-Schwartz inequality leads to $C(\pi) \leq d\sqrt{2d+1}$. We propose to find the permutation $\pi^*$ such that:
\[C(\pi^*)=\max_{\pi}C(\pi)\,.\]Again, this problem aims at finding permutations such that pairs of eigen vectors have a complex envelope which maximizes the energy along the span of $A$. Observe that this loss can be written as a separable sum of 2-entries losses. In this case, an exact solution can be obtained via a Blossom algorithm \cite{edmonds_1965} which runs in polynomial time. Note that this algorithm  combined with the eigen-decomposition procedure needs to be computed once, and it leads to a computational complexity of $\mathcal{O}(d
^3)$ in the worst case scenario. We then consider the matrix $\mathcal{F}=\{\mathcal{F}_i\}_{i\leq 2d+1}$ whose columns are defined by $\forall 1\leq i\leq d,$ 
\begin{equation}
    \begin{cases} \mathcal{F}_i&= e_{\pi^*[2i]}+\mathbf{j}e_{\pi^*[2i-1]}\,, \\\mathcal{F}_{2d+1-i}&=e_{\pi^*[2i]}-\mathbf{j}e_{\pi^ *[2i-1]}\,,\\
    \mathcal{F}_{2d+1}&=e_{2d+1}\,.
    \end{cases}
\end{equation}

Observe that if $i\leq d$, then  $\overline{\mathcal{F}_i}=\mathcal{F}_{2d+1-i}\in\mathbb{C}^{2d+1}$. We can then state the following proposition:
\begin{proposition}
The matrix $\mathcal{F}$ is unitary on $\mathbb{C}^{2d+1}$.
\end{proposition}
\begin{proof}
For simplicity, assume $\pi^*[i]=i$. Then, let $i$ s.t. $i\leq d$. Assume first that $j\leq d, j\neq i$, then $e_{2i}\perp e_{2j},e_{2i}\perp e_{2j-1}, e_{2i-1}\perp e_{2j}, e_{2i-1}\perp e_{2j-1}$, thus $\mathcal{F}_i \perp \mathcal{F}_j$ and also $\mathcal{F}_i \perp \overline{\mathcal{F}_j}$. Finally if $j=2d+1-i$,
\begin{align*}
    \langle \mathcal{F}_i, \overline{\mathcal{F}_i}\rangle&=\langle e_{2i}+\mathbf{j}e_{2i-1},e_{2i}-\mathbf{j}e_{2i-1}\rangle=\Vert e_{2i}\Vert^2-\Vert e_{2i-1}\Vert^2\\
&=0
\end{align*}\end{proof}
For illustration purpose, consider the graph $\mathcal{G}$ of a grid of length $2d+1$ with periodic boundary condition, an eigenbasis of its discrete Laplacian is clasically given, for $k\leq d, m\leq 2d+1$, by:
\begin{align}
&e_{2k-1}[m]= \sqrt{\frac{1}{2d+1}}\cos(\frac{\pi}{2d+1}(m-\frac{1}{2})2k)\,,\\
&e_{2k}[m]= \sqrt{\frac{1}{2d+1}}\sin(\frac{\pi}{2d+1}(m-\frac{1}{2})2k)\,,\\
&e_{2d+1}[m]=\frac{1}{\sqrt{2d+1}}\,.
\end{align}
In this case, we have the following lemma to derive the optimal pairing $\pi^*$:
\begin{lemma}
An optimal permutation $\pi^*$ is given by $\pi^*[n]=n$.
\end{lemma}
\begin{proof}
Introducing $Ax=\frac{1}{\sqrt{2d+1}}\sum_m x[m]$, from Cauchy Schwartz, under the constraint $\Vert x\Vert=1$, $Ax$ is maximal iff $x[m]=1,\forall m$. This is in particular true for $|x|[m]\triangleq |x[m]|$ if $x[m]=e^{\mathbf{j}\omega m}$ for some $\omega\in\mathbb{R}$, which is achieved by the pairing proposed in this Lemma.
\end{proof}

In this case, $\mathcal{F}_i[m]=\sqrt{\frac{1}{2d+1}}e^{\mathbf{j}\frac{\pi}{2d+1}(m-\frac{1}{2})2i}, i\leq d$. This thus justifies the terminology  Fourier Transform for $\mathcal{F}$ (up to a phase multiplication) as one can recover the Discrete Fourier Transform: our method has a natural interpretation in the Euclidean case. Pairing those eigen-vectors allows to introduce an asymetry between the real and imaginary part of our spectral operator, which will be useful and necessary for learning a complex unitary operator, that we discuss in the next section.

\subsection{Specifying the isometry layer per layer}\label{isometry}
\subsubsection{An energy preserving procedure}
We now describe our objective for specifying each operator $W_n$ at order $n$. The graph filtering operators $W_n$ that we consider consists of $K$ filters,  meaning that the eigenvalues $\{\hat W_n^k\}\subset \mathbb{C}^{2d+1}, k\leq K$ of each $W_n^k$ can be derived from $\mathcal{F}$ via:
\[\mathcal{F}^*W_n^k\mathcal{F}=\text{diag}(\hat W_n^k),\forall k\leq K\,.\]

With a slight abuse of notations when non-ambiguous, we might write $\text{diag}(\hat W_n^k)$ as $\hat W_n^k$. Let us write $x_1,...,x_p,...\in\mathbb{R}^{Q\times (2d+1)}$ our data points, where $Q\in\mathbb{N}$ is the number of input channels. For a signal $z$, consider the loss: \[\ell(W,z)=\Vert (\mathbf{I}-A)z\Vert^2-\Vert A|Wz|\Vert^2\,.\] 
If $W$ is  an isometry, this quantifies the energy preserved after an averaging $A$. We will only consider operators which are 1-Lipschitz and diagonalized by $\mathcal{F}$,  thus we consequently introduce:
\begin{equation}
\mathcal{C}=\{W; \Vert \{W,A\}\Vert\leq 1,\mathcal{F}^*W^k\mathcal{F}=\text{diag}(\hat W^k),\forall k\}\,.
\end{equation}
Our operator $W_n$ will be specified by the following minimization of the  empirical risk $L$:
\begin{equation}
W_n\triangleq\arg\min_{W\in\mathcal{C}} \sum_p \ell(W,U_{n}x_p)\label{loss} = \arg\min_{W\in\mathcal{C}} L(W)\,.
\end{equation}

Appendix \ref{app} proves that $L$ is concave in $W$, and is positive if  $\Vert \{W,A\}\Vert\leq 1$. It is very similar to a Procrustes problem~\cite{schonemann1966generalized}. Concave minimization has been well studied, and a global solution of a concave minimization over a convex set lays in the extremal points of this convex set \cite{horst1984global,rockafellar1970convex,doi:10.1137/1028106}. The next proposition characterizes the extremal point of $\mathcal{C}$, whose proof is defered to the Appendix.

\begin{proposition}
Let $\mathcal{S}$ the extremal points of $\mathcal{C}$, then $\mathcal{S}\subset  \{W,\Vert Wx\Vert =\Vert (\mathbf{I}-A)x\Vert,\forall x\in \mathbb{R}^{2d+1}\}$.
\end{proposition}

The Figure \ref{analytic} represents the spectrum of an operator $W$ learned from the small natural images of CIFAR-10. Remarkably, this operator is analytic, meaning that half of the frequency plane is set to 0 (up to a per-filter central symetry). This is natural because the analytic part of a filter is known to provide a smoother envelope, which is better captured by a low-pass filtering \cite{mallat1999wavelet}. In the settings of \citet{oyallon2018compressing}, this is quantified. Note also that the filters are localized and smooth in frequency, which indicates that the learned filters have efficiently used the topology of the frequency domain, without explicitely incorporating any specific \textit{a priori}.

\begin{figure}[ht]
\vskip 0.2in
\begin{center}
\centerline{\includegraphics[width=\columnwidth]{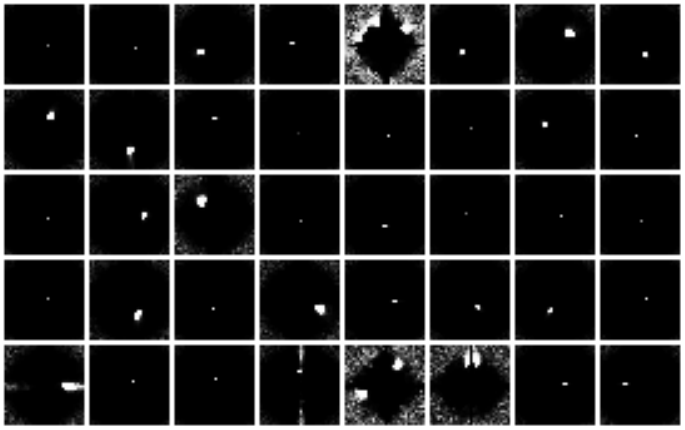}}
\caption{Spectral filters's modulus $|\hat W^k_1|$ of a first order learned operator. The 0 frequency corresponds to the center. An approximatively analytic transform is obtained by solving Eq. \eqref{loss} for small CIFAR-10 natural images. Observe that the filters are well localized with various frequency bandwidths, and most of them are analytic. Some filters have some noisy high frequencies, which are mainly due a poor conditioning: indeed signals' energy is concentrated in a disk around the central frequency.}
\label{analytic}
\end{center}
\vskip -0.2in
\end{figure}

For a general graph, the topology of the frequency index is in general unknown, meaning that designing an analytic operator is challenging. On the other hand, our criterion should enforce filters which  have a smooth modulus, i.e., which maximize the energy of the envelope $|Wx|$ along the span of $A$. In the next section, we discuss how to optimize Eq. \eqref{loss}.

\subsubsection{A projected gradient method}\label{proj}

We propose to minimize Eq. \eqref{loss} by a projected gradient descent as derived in \cite{hager2016projection}. Obtaining a global minimizer is a difficult task as explained in \citet{horst1984global}. In order to define our projection, we introduce a Littlewood-Paley identity \cite{stein1970topics} $\Gamma$, related to $W$ defined by:
\begin{equation*}
\begin{cases}
\Gamma_i = \sum_{k=1}^K |\hat W^k[2d+1-i]|^2+|\hat W^k[i]|^2\,, i\leq 2d\,,\\
\Gamma_{2d+1}=2\sum_{k=1}^K |\hat W^k[2d+1]|^2\,,
\end{cases}
\end{equation*}
as well as $\Lambda$ related to $\mathbf{I} - A$:

\begin{equation*}
\begin{cases}
\Lambda_i = 2-|\hat A[i]|^2-|\hat A[2d+1-i]|^2\,, i\leq 2d\,,\\
\Lambda_{2d+1}=2-2|\hat A[2d+1]|^2\,.
\end{cases}
\end{equation*}

We then define the diagonal matrix $\hat P$, whose diagonal is:
\begin{equation*}
\hat P_i = \begin{cases}
1 \text{, if } \Gamma_i\leq \Lambda_i\,,\\
\sqrt{\frac{\Lambda_i}{\Gamma_i}}\text{, otherwise.}
\end{cases}
\end{equation*}

Then the convex projection on $\mathcal{C}$ is given filter-wise by:

\[\text{proj}(W^k)=\mathcal{F}^*\hat W^k\hat P\mathcal{F}\,.\]

This leads to the following scheme, for a decreasing sequence of step size $\alpha_t$ and $L_t$ the loss at step $t$:
\[W^{t+1}=\text{proj}(W^t-\alpha_t \nabla_W L_t(W))\,,\]

with a random initialization $W_0$ such as a white noise. In our experiments, we typically use a stochastic gradient, thus the loss $L_t$ corresponds to the empirical loss over a randomly selected batch of samples. The next section provides numerical experiments which corroborate that we successfully optimized this operator to obtain a discriminative but smooth representation.

\section{Numerical experiments}\label{exp}
For each experiments, we combine an IGT representation with a linear SVM, as implemented by \citet{fan2008liblinear}. We select the order of the IGT such that more than 99\% of the energy is captured, leading at maximum to an order 2. We systematically report state-of-the-art performances in unsupervised spectral GCN settings. We note that in all our benchmarks, typical unsupervised representations are  shallow, similarly to our representation, which indicates that training deep unsupervised representations is a  challenging task. The regularization  of the SVM was cross-validated as $C=10^{-k},k=0,...,4$; no more than 3 runs have been done, without any intensive grid search. Also, as a sanity check, we experimented with  random $\mathcal{F}$ or $W_n$, leading systematically to substential drops in performances.
\subsection{Image classification}
In each vision experiments, we consider the Laplacian obtained from a regular grid and we follow \citet{henaff2015deep} combined with our Lemma \ref{mean}: we can explicitely consider the standard 2D Discrete Fourier Transform. Note that shuffling image's pixels (assuming that the Laplacian is shuffled consistantly) does not affect our algorithm: our method allows to recover the Euclidean grid structure without being explicitely incorporated, contrary to a 2D CNN.  We limited our experiments to a single layer, because it already captures most of the signals' energy. We compared our numerical performances against a Gabor Scattering Transform \cite{andreux2018kymatio} as well as a Haar Scattering Transform. In \citet{chen2014unsupervised}, two settings are considered: one for which the geometry is known (2D grid), and one for which it is not (no grid). Also in \citet{chen2014unsupervised}, an ensembling of models, combined with a supervised feature selection algorithm is used, as well as a Gaussian SVM. Instead, for the sake of comparison, we have re-run their code for a single model followed by a Linear SVM, which leads to a substantial drop in performances. In both settings, our operators have $K_1=40$ filters. The operator $W_1$ is learned via  SGD with batch size 64, for 5 epochs. We reduced an initial learning rate  of $1.0$ by $10$ at  iterations 500, 1000 and 1500. \citet{oyallon2015deep} and \citet{bruna2013invariant} found that averaging Scattering representations with a low-pass filter over a windows of length $2^J=2^3$ was optimal, thus we did not change this hyper parameter: if $\phi_J$ is a Gaussian filter of length $2^J$, we consider in those experiments:
\begin{equation}
Ax(u)=x\star\phi_J(u/2^J)\,.
\end{equation}
\subsubsection{CIFAR-10}
CIFAR-10 is a challenging dataset of small $32\times32$ colored images, which consists of $5\times10^4$ images for training and $10^4$ for testing. Table \ref{cifar10} reports our performances with a linear classifier.   Observe that our method improves by about 10\% the classification from the raw data. We also compare our work with supervised spectral methods, and we achieve similar performances without supervision. Despite incorporating more pointwise non-linearity, a Haar Transform performs substantially worse, which indicates that Haar features are not discriminative enough. Our spectral method leads to state-of-the-art performances, competitive with supervised methods.  By incorporating the Euclidean domain knowledge, a Gabor Scattering outperforms by $10\%$ the IGT, and adding some additional supervision and more non-linearity leads to the state of the art on CIFAR10 \cite{zagoruyko2016wide}.

\begin{table}[t]
\caption{Classification accuracies on CIFAR-10. \textit{Sup.} and \textit{Acc.} stand respectively for Supervision and Accuracy.}\label{cifar10}
\label{sample-table}
\vskip 0.15in
\begin{center}
\begin{small}
\begin{sc}
\begin{tabular}{lcccc}
\toprule
Method & Depth&Sup.& Acc. \\
\midrule
Raw data &-&$\times$&39.7\\
\hline
\hline
Spectral GCN&&&\\
\hline
\hline
IGT (ours) & 1&$\times$&\textbf{52.4}\\
Haar Scattering (no grid) & 2 &$\times$& 46.3\\
\cite{knyazev2019image}&3  &$\surd$& 50.6\\
\cite{7979525} &1&$\surd$&$\sim 52$\\
\hline
\hline
2D CNN\\
\hline
\hline
Gabor Scattering  &1  &$\times$& 64.9 \\
Haar Scattering (2D grid) & 4 &$\times$& 43.4\\
\cite{zagoruyko2016wide}  &40&$\surd$&\textbf{94.1}\\
\bottomrule
\end{tabular}
\end{sc}
\end{small}
\end{center}
\vskip -0.1in
\end{table}

\subsubsection{MNIST}
MNIST is a simple dataset of small $28\times28$  images, which consists of $6\times10^4$ images for training and $10^4$ for testing. Table \ref{mnist} reports the accuracy of our method. Again, our method outperforms unsupervised spectral representations: for instance, IGT outperforms \cite{zou2019graph} which defines a Scattering Transform based on graph wavelets. Adding some supervisions, such as in \citet{defferrard2016convolutional}, allows to obtain competitive performances with spatial convolutional methods \cite{bruna2013invariant}.

\begin{table}[t]
\caption{Classification accuracies on MNIST.}\label{mnist}
\label{sample-table}
\vskip 0.15in
\begin{center}
\begin{small}
\begin{sc}
\begin{tabular}{lcccc}
\toprule
Method & Depth&Sup.& Acc. \\
\midrule
Raw data &-&$\times$&93.8\\
\hline
\hline
GCN&&&\\
\hline
\hline
IGT (ours) & 1&$\times$&96.1\\
\cite{zou2019graph}&2&$\times$&95.6\\
Haar Scattering (no grid) & 6 &$\times$& 82.3\\
\cite{defferrard2016convolutional}& 2&$\surd$&\textbf{99.1}\\
\hline
\hline
2D CNN\\
\hline
\hline
Haar Scattering (2D grid) & 4 &$\times$& 88.6\\
Gabor Scattering  &1  &$\times$& 98.6 \\
\cite{defferrard2016convolutional}&2  &$\surd$& \textbf{99.3} \\
\bottomrule
\end{tabular}
\end{sc}
\end{small}
\end{center}
\vskip -0.1in
\end{table}

\subsection{Action prediction}
We now consider several 3D skeletons datasets, whose objective is to predict an action from a sequence of frame.  For each dataset, a (handcrafted) skeleton represented as a graph is provided, based on human body connectivity, whose nodes are the coordinates of some human body parts (not images). Here, we preprocess our datasets using the representation proposed by \citet{mazari2019mlgcn}, which consists of a temporal barycenter of each node's coordinates taken along non-overlapping windows of equal time length. We note that our goal is not to propose a new better pre-processing method than the other works we compared to, yet to improve the initial features, thus we reported  the raw data accuracy.
\subsubsection{SBU}
Eeach  SBU sample describes a two person interaction, and SBU contains 230 sequences and 8 classes (6,614
frames). The corresponding graph has 30 nodes. The accuracy is reported as the mean of the accuracies of a 5-fold procedure. We used an order 2 IGT, with $K_1=K_2=30$ filters for each operator. We train our operators for 5 epochs, with a batch size of 64, an initial learning rate of 1.0, dropped by 10 at the iterations 10, 20 and 30. Table \ref{sbu} reports the accuracy of various supervised and unsupervised method on SBU. An IGT  improves by about 2\% a linear classifier on the raw data. Furthermore, our method achieves similar performances compared to supervised methods, while outperforming unsupervised representations.
\begin{table}[t]
\caption{Accuracies on SBU, via a standard 5-fold procedure.}\label{sbu}
\label{sample-table}
\vskip 0.15in
\begin{center}
\begin{small}
\begin{sc}
\begin{tabular}{lcccc}
\toprule
Method & Depth& Acc. \\
\midrule
Raw data  &-&92.7\\
\hline
\hline
Unsupervised\\
\hline
\hline
IGT (ours) & 1&91.3$\pm 1.0$\\
IGT (ours) & 2&\textbf{94.5$\pm 1.0$}\\
\cite{kacem2018novel} & - &93.7\\
\hline
\hline
Supervised\\
\hline
\hline
\cite{wu2019simplifying}\footnote{The result is obtaineds from a re-implementation of~\cite{mazari2019mlgcn}}&1&96.0\\
\cite{mazari2019mlgcn}& 1 & \textbf{98.6}\\
\bottomrule
\end{tabular}
\end{sc}
\end{small}
\end{center}
\vskip -0.1in
\end{table}

\subsubsection{NTU}
NTU is a challenging dataset for large scale human action analysis \cite{shahroudy2016ntu}, with 60 different classes and 56880 samples, corresponding to 40 subjects and 80 different views. The corresponding graph has 50 nodes. Two procedures allow to report the accuracy. In cross-subject evaluation,  40 subjects are split into training and testing groups, consisting of 20 subjects such that the training and testing sets have respectively 40,320 and 16,560 samples. For cross-view evaluation, the samples are split according to different cameras view, such that the training and testing sets have respectively 37,920 and 18,960 samples.

We use $K_1=10$ and $K_2=5$ filters respectively for  our two learned operators. We trained via SGD our representation, with a batch size of 64, an initial learning rate of $1.0$ being dropped by $10$ at iterations 100, 200 and 300. Table \ref{ntu} reports the accuracies for various unsupervised and supervised methods. First, observe that our method improves by about 30\% a linear SVM trained on the raw features. It also outperforms all the unsupervised methods, yet a sigificant gap of 30\% exists with supervised algorithm.

Here, we note that a global invariant to permutations was not required, thus we did not average our representation. In this case, a linear invariant is obtained from a linear SVM, which can freely adjust the degree of invariance to the supervised task which is considered, yet it leads to an extra-computational cost. Table \ref{ablation} corresponds to an ablation of our method and indicates that here, not averaging our representation improves accuracies. Note also that in this case, first and second orders perform similarly: it indicates that the second order doesn't recover more informative attributes yet this doesn't invalidate any claims done. It also leads to a substential increasing of dimension. However, with an averaging, the second order brings a significant improvement over the first order because it recovers more information. Without averaging our representation, the accuracies vary by less than 1\%.

\begin{table}[t]
\caption{Classification accuracies on NTU, with a final linear classifier.}\label{ntu}
\label{sample-table}
\vskip 0.15in
\begin{center}
\begin{small}
\begin{sc}
\begin{tabular}{lcccc}
\toprule
Method & Depth& View & Sub \\
\midrule
Raw data &-&22.9&31.9\\
\hline
\hline
Unsupervised\\
\hline
\hline
IGT (ours) & 1&54.6&\textbf{60.5}\\
IGT (ours) & 2&\textbf{55.6}&59.9\\
\cite{evangelidis2014skeletal} & 2&41.4&38.6\\
\cite{vemulapalli2014human} & -&52.8&50.1\\
\hline
\hline
Supervised\\
\hline
\hline
\cite{li2019symbiotic}  & $\geq 2$& 90.1&96.4\\
\bottomrule
\end{tabular}
\end{sc}
\end{small}
\end{center}
\vskip -0.1in
\end{table}

\begin{table}[t]
\caption{Ablation experiments on NTU, with IGT.}\label{ablation}
\label{sample-table}
\vskip 0.15in
\begin{center}
\begin{small}
\begin{sc}
\begin{tabular}{lcccc}
\toprule
Order&Averaging& View & Sub \\
\midrule
1&$\surd$&39.7&44.0\\
2&$\surd$&49.3&52.9\\
\midrule
1&$\times$&54.6&\textbf{60.5}\\
2&$\times$&\textbf{55.6}&59.9\\
\bottomrule
\end{tabular}
\end{sc}
\end{small}
\end{center}
\vskip -0.1in
\end{table}

\subsection{Community detection}
We reproduce the experiments of \citet{gama2019stability} using their provided source code, and we compare our representation with a Graph Scattering Transform (GST) using various mother wavelets. In all our experiments, we used a single order IGT, and our operator is learned with a SGD with constant step size of $10^{-3}$ and batch size of 64. We followed the same evaluation procedure as \citet{gama2019stability}. Our experiments suggest that selecting a linear operator according to a smoothness criterion can improve the numerical performances compared to a wavelet transform, which is stable to deformations.
\subsubsection{Authorship attribution}
This dataset consists of a graph with 188 nodes representing a bag-of-words for some collection of texts. The objective is to decide if a writer is the author of a given text. The benchmarking of this dataset consists of reporting the accuracy given a number of training sample. Observe on Figure \ref{author}, that  IGT systematically outperforms \citet{gama2019stability} for each training size.
\begin{figure}[ht]
\vskip 0.2in
\begin{center}
\centerline{\includegraphics[width=\columnwidth]{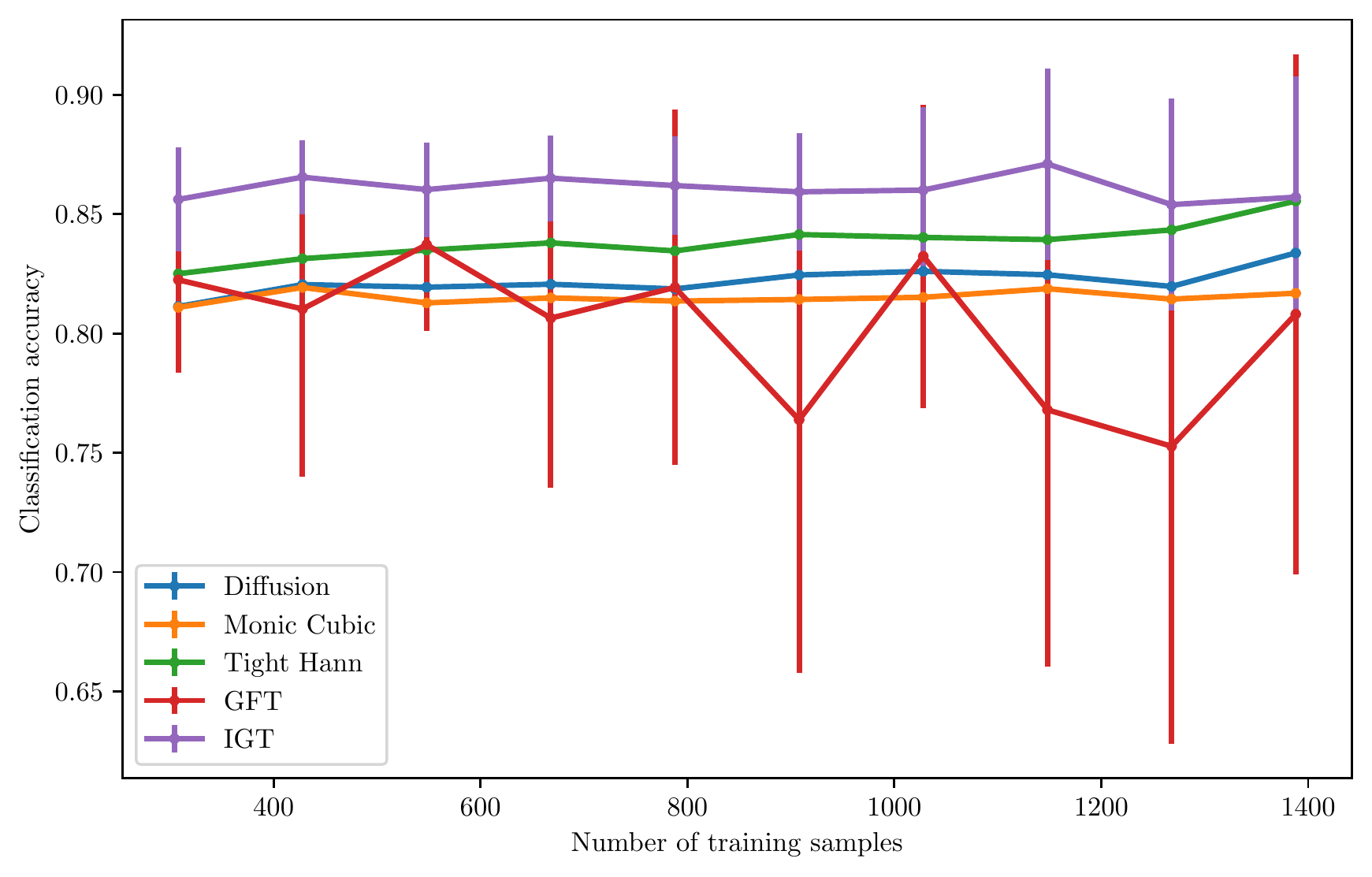}}
\caption{Authorship attribution.  We compare our IGT with various mother wavelets proposed in \citet{gama2018diffusion}. The accuracy is a function of the number of training samples.}
\label{author}
\end{center}
\vskip -0.2in
\end{figure}
\subsubsection{Facebook graph}
The dataset consists of a synthetic 234 nodes graph modeling some Facebook interactions. A diffusion process is initiated at some node, and the objective is to determine which community this original node belongs to. In order to make this dataset challenging, a fraction of the edge is dropped and the classification accuracy is reported for various probability of edge failure. $2\times10^3$ points are used from training and $2\times10^2$ points for testing. Figure \ref{fb} reports our performances. Obtaining 100\% on the test set, our method solves this dataset, which clearly outperforms  GST introduced in \citet{gama2018diffusion} for each wavelet family.
\begin{figure}[ht]
\vskip 0.2in
\begin{center}
\centerline{\includegraphics[width=\columnwidth]{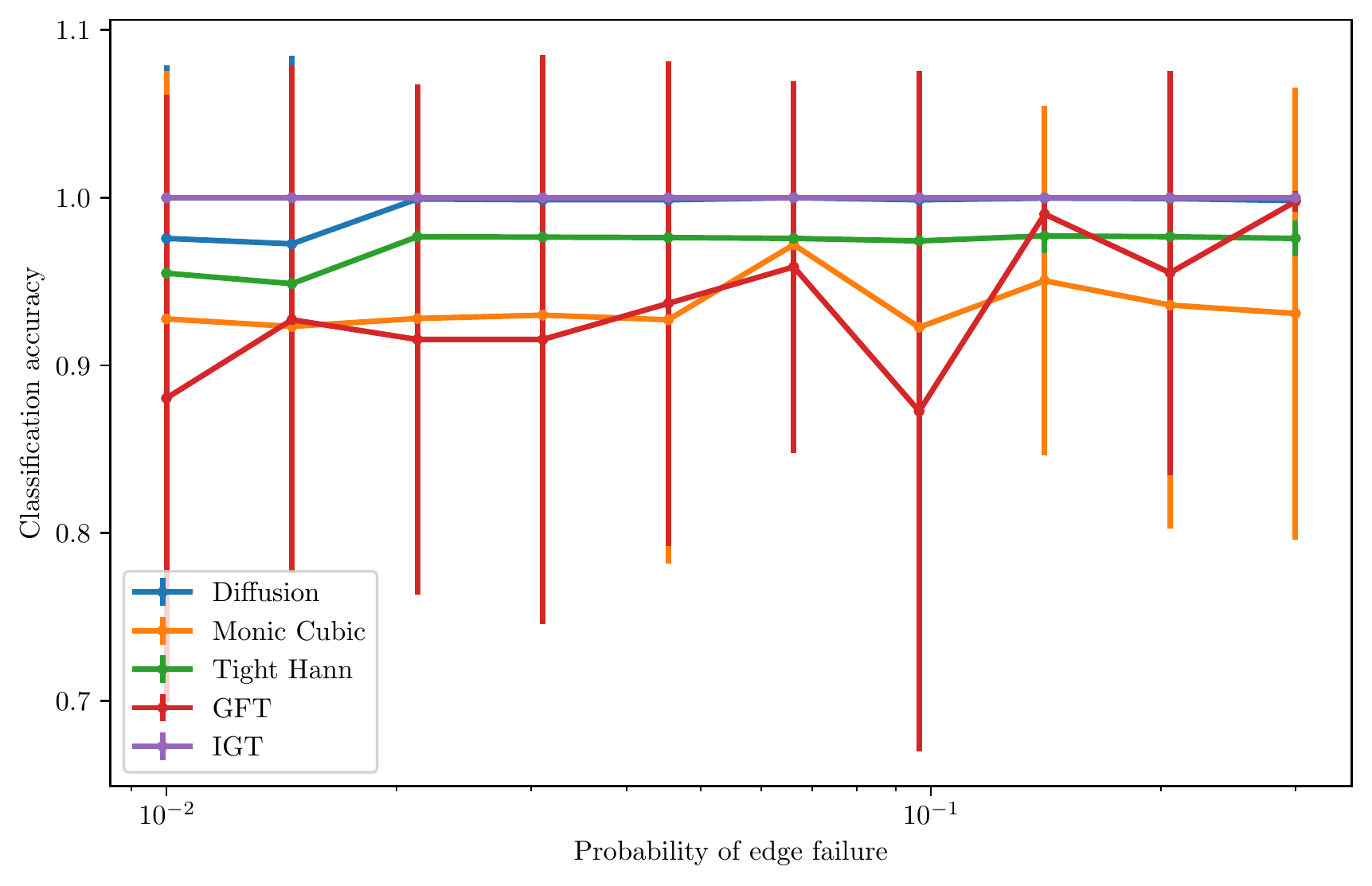}}
\caption{Facebook graph. We compare our IGT with various mother wavelets proposed in \citet{gama2018diffusion}.  The classification accuracy is reported as well as its variance, as a function of the fraction of edges dropped.}
\label{fb}
\end{center}
\vskip -0.2in
\end{figure}

\section{Conclusion}
In our work, we introduced the Inteferometric Graph Transform, which is an unsupervised, generic and interpretable representation that is guaranteed to obtain smooth features. We introduced a complex unitary transform for graphs analog to a Fourier transform. Thanks to our concave optimization procedure motivated by invariance and energy preservation considerations, we obtain performances  at the level of the state of the art on many various complex benchmarks. In vision settings, we observe that our method obtains analytic and well structured operators, which is surprising.

In a future work, we would like to extend this method to hybrid models, combining IGT and deep supervised GCN models, as done in \citet{oyallon2018scattering,Oyallon_2017_ICCV} for natural images on a regular grid. Another question which is still open, is to understand if it would be possible to provide a low dimensional mapping~\cite{jacobsen2017multiscale} of our spectral basis, similar to the index of a $N$-dimensional Fourier basis.

\paragraph*{Acknowledgement:}The author would like to thank Mathieu Andreux, Eugene Belilovsky, Bogdan Cirstea, Thomas Pumir, John Zarka for proofreading this manuscript, as well as Lucas Silve and Ahmed Mazari for helpful discussions related to the experiments. The author was supported by a GPU donation from NVIDIA.

\bibliography{example_paper}
\bibliographystyle{icml2020}

\clearpage
\newpage
\appendix
\newpage
\section{Proof that $L$ is concave and positive.}\label{app}
 
 We will use the notations previously introduced as well as:
 \[z_p = U_nx_p\,.\]
 As $L(W)=\sum_p \ell(W,z_p)$, we will simply study $\forall z$,
 \[W\rightarrow \ell(W,z)=\Vert (\mathbf{I}-A)z\Vert^2-\Vert A|Wz|\Vert^2\,.\]
 
 Observe first that if $\Vert \{W,A\}\Vert\leq 1$, then $\Vert A\vert\leq 1$, and:
 \[\Vert Wz\Vert \leq \Vert (\mathbf{I}-A)z\Vert\]
 Thus,
 \[\Vert |Wz|\Vert \leq \Vert (\mathbf{I}-A)z\Vert\]
 and:
 \[\Vert A|Wz|\Vert \leq \Vert (\mathbf{I}-A)z\Vert\,.\]
 Consequently, $\ell(W,z)\geq 0,\forall z,\forall W\in\mathcal{C}$. Furthermore, let $W_1,W_2\in\mathcal{C}$ two operators and $0\leq \lambda\leq 1$. Then:
 \[|\big(\lambda W_1+(1-\lambda)W_2\big)z|\leq \lambda |W_1|z+(1-\lambda)|W_2|z\]
 where for $x\in \mathbb{R}^n$, $x\geq 0$ iff $x_i\geq 0$. If $Ax>0$ when $x>0$, then:
 \[A|\big(\lambda W_1+(1-\lambda)W_2\big)z|\leq \lambda A|W_1|z+(1-\lambda)A|W_2|z|\,,\]
 which implies (as all coordinates are non negative):
 \[\Vert A|\big(\lambda W_1+(1-\lambda)W_2\big)z|\Vert ^2\leq \Vert \lambda A|W_1|z+(1-\lambda)A|W_2|z|\Vert^2\,,\]
 yet one can use the fact that $z\rightarrow \Vert z\Vert^2$ is convex to conclude. Thus, $W\rightarrow \ell(W,z)$
 is convex in $W$.
 \section{Proof of Proposition 3.5}
 \begin{proof}
Observe that $\mathcal{F}$ linearly conjugates $\mathcal{C}$ to $\{\hat W\in \mathbb{C}^{(2d+1)\times K}, \sum_{k=1}^K |\hat W^k[i]|^2+|W^k[2d+1-i]|^2+|\hat A[i]|^2+|\hat A[2d+1-i]|^2\leq 1, \forall i\leq d,  \sum_{k=1}^K |W^k[2d+1]|^2+|\hat A[2d+1]|^2\leq1\}$. The extremal points of the latter are simply $\mathcal{S}'=\{\hat W\in \mathbb{C}^{(2d+1)\times K}, \sum_{k=1}^K |\hat W^k[i]|^2+|W^k[2d+1-i]|^2+|\hat A[i]|^2+|\hat A[2d+1-i]|^2=1,  \forall i\leq d, \sum_{k=1}^K |W^k[2d+1]|^2+|\hat A[2d+1]|^2=1\}$, which is conjugated by $\mathcal{F}^*$ to $\mathcal{S}$. But $\mathcal{S}'$ corresponds to the spectrum of an isometry, leading to the conclusion.
\end{proof}

\end{document}